\documentclass[11pt]{article} 
\pdfoutput=1
\usepackage[margin=1in]{geometry}

\usepackage{natbib}
\usepackage[utf8]{inputenc} 
\usepackage[T1]{fontenc}    
\usepackage{hyperref}       
\usepackage{url}            
\usepackage{booktabs}       
\usepackage{amsfonts}       
\usepackage{nicefrac}       
\usepackage{microtype}      
\usepackage{xcolor}         
\usepackage{amsmath}
\usepackage{amsthm}
\usepackage{graphicx}

\newtheorem{theorem}{Theorem}[section]

\newtheorem{lemma}[theorem]{Lemma}

\newtheorem{claim}[theorem]{Claim}
\theoremstyle{definition}
\newtheorem{definition}[theorem]{Definition}
\theoremstyle{remark}

\usepackage[linesnumbered,ruled]{algorithm2e}
\SetKwInput{KwInput}{Input}
\SetKwInput{KwOutput}{Output}
\newcommand\A{\mathcal{A}}

\newcommand\C{\mathcal{C}}
\newcommand\D{\mathcal{D}}
\newcommand\I{\mathbb{I}}
\newcommand\X{\mathcal{X}}
\newcommand\err{\mathrm{err}}
\newcommand\dis{\mathrm{dis}}
\newcommand\vc{\mathtt{VC}}
\newcommand\Rel{\A_{\mathrm{Relabel}}}
\newcommand\Agn{\A_{\mathrm{Agnostic}}}
\newcommand\Aux{\A_{\mathrm{Auxiliary}}}

\title{Private Realizable-to-Agnostic Transformation with Near-Optimal Sample Complexity}
\author{Bo Li\thanks{Department of Computer Science and Engineering, HKUST. \texttt{bli@cse.ust.hk}.}
\and 
Wei Wang\thanks{Department of Computer Science and Engineering, HKUST. \texttt{weiwa@cse.ust.hk}.}
\and
Peng Ye\thanks{Department of Computer Science and Engineering, HKUST. \texttt{pyeac@connect.ust.hk}.}}

\begin{document}

\maketitle

\begin{abstract}
  The realizable-to-agnostic transformation~\citep{beimel2014learning,alon2020closure} provides a general mechanism to convert a private learner in the realizable setting (where the examples are labeled by some function in the concept class) to a private learner in the agnostic setting (where no assumptions are imposed on the data). Specifically, for any concept class $\mathcal{C}$ and error parameter $\alpha$, a private realizable learner for $\mathcal{C}$ can be transformed into a private agnostic learner while only increasing the sample complexity by $\widetilde{O}(\mathrm{VC}(\mathcal{C})/\alpha^2)$, which is essentially tight assuming a constant privacy parameter $\varepsilon = \Theta(1)$. However, when $\varepsilon$ can be arbitrary, one has to apply the standard privacy-amplification-by-subsampling technique~\citep{kasiviswanathan2011can}, resulting in a suboptimal extra sample complexity of $\widetilde{O}(\mathrm{VC}(\mathcal{C})/\alpha^2\varepsilon)$ that involves a $1/\varepsilon$ factor.
  
  In this work, we give an improved construction that eliminates the dependence on $\varepsilon$, thereby achieving a near-optimal extra sample complexity of $\widetilde{O}(\mathrm{VC}(\mathcal{C})/\alpha^2)$ for any $\varepsilon\le 1$. Moreover, our result reveals that in private agnostic learning, the privacy cost is only significant for the realizable part. We also leverage our technique to obtain a nearly tight sample complexity bound for the private prediction problem, resolving an open question posed by~\citet{dwork2018privacy} and~\citet{dagan2020pac}.%
\end{abstract}

\section{Introduction}

Differential privacy (DP)~\citep{dwork2006calibrating,dwork2006our} has emerged as a popular notion for quantifying the disclosure of individual information and has been widely deployed to protect personal privacy~\citep{app2017apple,abowd2018us}. Informally, a randomized algorithm is said to be differentially private if changing a single input item does not significantly affect the output distribution. As a consequence, it safeguards against data inference through the algorithm's output.

Machine learning algorithms are usually trained on datasets that contain sensitive information, necessitating the need for privacy protection. The intersection of differential privacy and machine learning was first explored by~\citet{kasiviswanathan2011can}, who introduced \emph{private learning} by integrating DP with two foundational learning models: probably approximately correct (PAC) learning~\citep{vapnik1971uniform,valiant1984theory} and agnostic learning~\citep{haussler1992decision,kearns1994toward}. The former assumes the data points are labeled by some function in a given concept class $\C$ and requires the learner to output a hypothesis with an error close to $0$. This setting is often referred to as the \emph{realizable} setting. In contrast, the latter is termed the \emph{agnostic} setting, which can be seen as an extension of the realizable setting that removes the assumption on the existence of a labeling function. In agnostic learning, the output hypothesis is required to have an error close to optimal by any concept in $\C$, which might be much larger than $0$.

It turns out that the two settings are closely related: any private PAC learner can be transformed into a private agnostic learner via the realizable-to-agnostic transformation~\citep{beimel2014learning,alon2020closure}. More formally, given a private PAC learner for concept class $\C$, the transformation produces a private agnostic learner for $\C$ with an increase of $\widetilde{O}(\vc(\C)/\alpha^2)$ in the sample complexity, where $\alpha$ is the error parameter. Such an increase is optimal since it matches the lower bound on (non-private) agnostic learning~\citep{simon1996general}. However, this transformation results in an algorithm with a constant privacy parameter $\varepsilon=\Theta(1)$. If one pursues an arbitrary privacy level, the privacy-amplification-by-subsampling technique~\citep{kasiviswanathan2011can} has to be applied. This raises the extra sample complexity to $\widetilde{O}(\vc(\C)/\alpha^2\varepsilon)$, incorporating an undesirable $1/\varepsilon$ factor. It is natural to ask if we can remove the $1/\varepsilon$ factor and still achieve an optimal extra sample complexity when $\varepsilon$ is extremely small.

While private learning mandates that the entire output hypothesis must preserve privacy, this requirement might be excessively stringent. It has been shown that several concept classes, which can be easily learned in the non-private setting, pose significant challenges under differential privacy constraints~\citep{beimel2010bounds,feldman2014sample,bun2015differentially,alon2019private}. To bypass these hardness results,~\citet{dwork2018privacy} introduced the problem of \emph{private prediction}. This framework is particularly relevant when the complete model, trained on sensitive data, cannot be fully released to users (who might be potential adversaries). Instead, users submit queries consisting of unlabeled data points, and the system provides predictions on these points while ensuring privacy. In the realizable setting, they showed that the sample complexity of answering a single query privately is $\widetilde{\Theta}(\vc(\C)/\varepsilon\alpha)$, which is notably lower than that of private learning~\citep{beimel2019characterizing,alon2019private,bun2020equivalence,ghazi2021sample}. In the more challenging agnostic setting, the initial upper bound for sample complexity was $\widetilde{O}(\vc(\C)/\alpha^3\varepsilon)$. This was subsequently refined by~\citet{dagan2020pac} to $\widetilde{O}(\min(\vc(\C)/\alpha^2\varepsilon, \vc(\C)^2/\alpha\varepsilon) + \vc(\C)/\alpha^2)$. Despite these improvements, there remains a gap between this upper bound and the $\widetilde{\Omega}(\vc(\C)/\alpha\varepsilon +\vc(\C)/\alpha^2)$ lower bound.

\subsection{Results}

Our main contribution is a transformation that converts any realizable learning algorithm to an agnostic learning algorithm under $(\varepsilon,\delta)$-differential privacy. Remarkably, this transformation only increases the sample complexity asymptotically by $\widetilde{O}(\vc(\C)/\alpha^2)$ for any $\varepsilon \le 1$. Such an extra sample complexity is near-optimal as it matches the $\widetilde{\Omega}(\vc(\C)/\alpha^2)$ lower bound on agnostic learning even without privacy~\citep{simon1996general}.

\begin{theorem}[Informal Version of Theorem~\ref{thm:trans}]
\label{thm:main}
    An $(\varepsilon, \delta)$-differentially private realizable learner for $\C$ with error $\alpha$ and with sample complexity $m$ can be transformed into an $(\varepsilon,\delta)$-differentially private agnostic learner for $\C$ with excess error $O(\alpha)$ and with sample complexity $\widetilde{O}(m + \vc(\C)/\alpha^2)$.
\end{theorem}

Our methodology builds upon the foundational transformation proposed by~\citet{beimel2014learning}, which was originally limited to only achieving a constant level of privacy. We observe that directly applying the privacy-amplification-by-subsampling method~\citep{kasiviswanathan2011can} suffers from a $1/\varepsilon$ blow-up because it runs the transformation only on a subsampled dataset whose size is approximately $\varepsilon$ of the input size and naively discards unsampled data points. To effectively utilize all data points, we design a novel score function that estimates the generalization error using the entire dataset rather than only the subsampled dataset while still enjoying amplification of privacy, thus avoiding the privacy cost incurred by previous methods. Additionally, we adopt a technique due to~\citet{alon2020closure} to accommodate improper learners.

We also obtain improved results for the private prediction problem in the agnostic setting by applying our transformation to a private prediction algorithm in the realizable setting due to~\citet{dwork2018privacy} with some mild modification. The result is demonstrated as follows.

\begin{theorem}
\label{thm:pred_intro}
    There is an $\varepsilon$-differentially private prediction algorithm that $(\alpha,\beta)$-agnostically learns $\C$ using $\widetilde{O}(\vc(\C)/\alpha\varepsilon + \vc(\C)/\alpha^2)$ samples.
\end{theorem}

This upper bound is tight up to logarithmic factors. In fact, a matching lower bound can be derived by combining a lower bound of $\Omega(\vc(\C)/\alpha\varepsilon)$ on the sample complexity of private prediction in the realizable setting established by~\citet{dwork2018privacy} with an $\widetilde{\Omega}(\vc(\C)/\alpha^2)$ lower bound for agnostic learning~\citep{simon1996general}.

\subsection{Related Work}

\paragraph{Realizable-to-agnostic transformation:} The general approach of transforming realizable learners to agnostic learners under differential privacy was first introduced by~\citet{beimel2014learning}. Their method, however, is only applicable to proper learners. This limitation was later addressed by~\citet{alon2020closure}, who extended the applicability to improper learners using the generalization property of DP. Despite being general, their methods only satisfy a constant level of privacy. The connection between realizable and agnostic learning under privacy was also investigated by~\citet{hopkins2022realizable}. However, their reduction only works for a relaxation of private learning called semi-private learning, where the algorithms have access to a public unlabeled dataset, and cannot be applied to private learning, which treats the entire input dataset as private and does not rely on any public data points. The optimal sample complexity of private agnostic learning was also considered by~\cite{li2024improved} under pure differential privacy. They provided an algorithm building upon the pure private realizable learner in~\citep{beimel2019characterizing}. Nevertheless, the algorithm is not a general transformation and does not work for approximate differential privacy.

\paragraph{Private prediction:} The study of private prediction was pioneered by~\citet{dwork2018privacy}, who established a near-optimal sample complexity of $\widetilde{\Theta}(\vc(\C)/\alpha\varepsilon)$ for the realizable setting. In the agnostic setting, they gave a suboptimal upper bound of $\widetilde{O}(\vc(\C)/\alpha^3\varepsilon)$ for any general concept class $\C$. When $\C$ is the class of unions of intervals, they presented an algorithm with an expected excess error of $\alpha$ using $\widetilde{O}(\vc(\C)/\alpha\varepsilon + \vc(\C)/\alpha^2)$ samples, which is nearly optimal with a constant success probability. The upper bound for general concept classes was further improved to $\widetilde{O}(\min(\vc(\C)/\alpha^2\varepsilon, \vc(\C)^2/\alpha\varepsilon) + \vc(\C)/\alpha^2)$ by~\citet{dagan2020pac}. Similar to our method, their algorithm combines the realizable-to-agnostic transformation~\citep{beimel2014learning} and the privacy-amplification-by-subsampling technique~\citep{kasiviswanathan2011can} with some modification. However, the use of a VC argument in their proof introduces an extra multiplicative factor of $\vc(\C)$, rendering their bound suboptimal. 
\section{Preliminaries}

We start with some notations. Let $\X$ be an arbitrary domain. A concept/hypothesis is a function that labels each $x\in\X$ by either $0$ or $1$. A concept/hypothesis class is a set of concepts/hypotheses. We use $\D$ to denote a distribution over $\X\times\{0, 1\}$, and $\D_\X$ to denote its marginal distribution over $\X$. Let $S=\{(x_1,y_1),\dots,(x_n,y_n)\}\in(\X\times \{0, 1\})^n$ be a dataset consisting of $n$ data points. The corresponding unlabeled dataset is denoted as $S_\X = \{x_1,\dots,x_n\}$. For two hypotheses $h_1$ and $h_2$, we define $h_1\oplus h_2$ as a hypothesis such that $(h_1\oplus h_2)(x) = \I[h_1(x)\neq h_2(x)]$.

Given a hypothesis $h$, the \emph{generalization error} of $h$ with respect to a distribution $\D$ is defined as $\err_{\D}(h) = \Pr_{(x,y)\sim\D}[h(x)\neq y]$. The \emph{empirical error} of $h$ with respect to a dataset $S$ is defined as $\err_{S}(h) = \frac{1}{n}\sum_{i=1}^n \I[h(x_i)\neq y_i]$. 

For two hypotheses $h_1$ and $h_2$, the \emph{generalization disagreement} between $h_1$ and $h_2$ with respect to a distribution $\D_\X$ is defined as $\dis_{\D_\X}(h_1,h_2) = \Pr_{x\sim\D_\X}[h_1(x)\neq h_2(x)]$. The \emph{empirical disagreement} between $h_1$ and $h_2$ with respect to an unlabeled dataset $S_\X$ is defined as $\dis_{S_\X}(h_1,h_2) = \frac{1}{n}\sum_{i=1}^n \I[h_1(x_i)\neq h_2(x_i)]$.

In the PAC learning framework, the learning algorithm receives as input a dataset sampled according to some underlying distribution $\D$, where the data points are labeled by some concept $c\in\C$. The objective is to output a hypothesis $h$ with low generalization error $\err_\D(h)$.

\begin{definition}[PAC Learning~\citep{valiant1984theory}]
    We say a learning algorithm $\A$ is an $(\alpha,\beta)$-PAC learner for concept class $\C$ with sample complexity $m$ if for any distribution $\D$ over $\X\times\{0, 1\}$ such that $\Pr_{(x,y)\sim\D}[c(x)=y] = 1$ for some $c\in\C$, it takes a dataset $S = \{(x_1,y_1),\dots,(x_m,y_m)\}$ as input, where each $(x_i,y_i)$ is drawn i.i.d. from $\D$, and outputs a hypothesis $h$ satisfying
    \begin{equation*}
        \Pr[\err_\D(h) \le \alpha ] \ge 1 - \beta,
    \end{equation*}
    where the probability is taken over the random generation of $S$ and the random coins of $\A$.
\end{definition}

PAC learning focuses on the realizable case, which assumes that the underlying distribution $\D$ is labeled by some concept $c\in\C$. In contrast, agnostic learning~\citep{haussler1992decision,kearns1994toward} does not impose any assumptions on the distribution $\D$. Instead, the goal is to identify a hypothesis whose generalization error is close to that of the best possible concept in $\C$.

\begin{definition}[Agnostic Learning]
    We say a learning algorithm $\A$ is an $(\alpha,\beta)$-agnostic learner for concept class $\C$ with sample complexity $m$ if for any distribution $\D$ over $\X\times\{0, 1\}$, it takes as input a dataset $S = \{(x_1,y_1),\dots,(x_m,y_m)\}$, where each $(x_i,y_i)$ is drawn i.i.d. from $\D$, and outputs a hypothesis $h$ satisfying
    \begin{equation*}
        \Pr[\err_\D(h) \le \inf_{c\in\C}\err_\D(c) +\alpha] \ge 1 - \beta,
    \end{equation*}
    where the probability is taken over the random generation of $S$ and the random coins of $\A$.
\end{definition}

In PAC and agnostic learning, if the learner $\A$ always produces a hypothesis that is a concept in $\C$, then we say $\A$ is a \emph{proper} learner. Otherwise, we say $\A$ is an \emph{improper} learner.

We next introduce some useful tools and results from learning theory. 

\begin{definition}[The Growth Function]
    Let $S_\X = (x_1,\dots,x_n)$ be an unlabeled dataset of size $n$. The set of all dichotomies on $S_\X$ realized by $\C$ is denoted by 
    \begin{equation*}
        \Pi_{\C}(S_\X) = \left\{\{(x_1,c(x_1)),\dots,(x_n,c(x_n))\}\mid c\in \C\right\}.
    \end{equation*}
    The growth function of $\C$ is defined as $
        \Pi_\C(n) = \max_{S_\X\in\X^n}\lvert\Pi_\C(S_\X)\rvert$.
\end{definition}

The Vapnik-Chervonenkis (VC) dimension~\citep{vapnik1971uniform} of a concept class $\C$ is defined as the largest number $d$ such that $\Pi_\C(d) = 2^d$ (or infinity, if no such maximum exists), denoted by $\vc(\C)$. Sauer's lemma~\citep{sauer1972density} states that the number of dichotomies is polynomially bounded if $\C$ has a finite VC dimension.

\begin{lemma}[Sauer's Lemma]
\label{lem:sauer}
    For any $n\ge \vc(\C)$, we have $\Pi_\C(n)\le \left(\frac{en}{\vc(\C)}\right)^{\vc(\C)}$.
\end{lemma}

The following realizable generalization bound~\citep{vapnik1971uniform,blumer1989learnability} relates the generalization disagreement and the empirical disagreement simultaneously for every pair of concepts: if one is small, then the other will also be small.

\begin{lemma}[Realizable Generalization Bound]
\label{lem:rea_gen}
    Let $\C$ be a concept class and $\D_\X$ be a distribution over $\X$. Suppose $S_\X=\{x_1,\dots,x_n\}$, where each $x_i$ is drawn i.i.d. from $\D_\X$. Define the following two events:
    \begin{itemize}
        \item $E_1$: For any $h_1, h_2\in\C$ such that $\dis_{S_\X}(h_1,h_2) \le \alpha$, it holds that $\dis_{\D_\X}(h_1,h_2)\le 2\alpha$.
        \item $E_2$: For any $h_1, h_2\in\C$ such that $\dis_{\D_\X}(h_1,h_2) \le \alpha$, it holds that $\dis_{S_\X}(h_1,h_2)\le 2\alpha$.
    \end{itemize}
    Then we have
    \begin{equation*}
        \Pr[E_1\cap E_2]\ge 1-\beta
    \end{equation*}
    given that
    \begin{equation*}
        n\ge C\cdot \frac{\vc(\C)\ln(1/\alpha) + \ln(1/\beta)}{\alpha}
    \end{equation*}
    for some universal constant $C$.
\end{lemma}

We also have the following agnostic generalization bound~\citep{talagrand1994sharper}, which ensures that for every concept $c\in \C$, its generalization error and empirical error are close. Unlike the realizable generalization bound, the agnostic generalization bound does not require the error to be small. However, this relaxation increases the sample complexity by roughly a factor of $1/\alpha$.

\begin{lemma}[Agnostic Generalization Bound]
\label{lem:agn_gen}
    Let $\C$ be a concept class and $\D$ be a distribution over $\X\times\{0, 1\}$. Suppose $S=\{(x_1, y_1),\dots,(x_n,y_n)\}$, where each $(x_i, y_i)$ is drawn i.i.d. from $\D$. Then we have
    \begin{equation*}
        \Pr[\forall c\in\C,\lvert\err_S(c) - \err_\D(c)\rvert \le \alpha]\ge 1-\beta
    \end{equation*}
    given that 
    \begin{equation*}
        n\ge C\cdot \frac{\vc(\C) + \ln(1/\beta)}{\alpha^2}
    \end{equation*}
    for some universal constant $C$.
\end{lemma}

In this work, we consider learning algorithms that preserve differential privacy. We say that two datasets $S_1$ and $S_2$ are neighboring if they differ by a single entry. A private algorithm is required to produce similar outputs for every pair of neighboring datasets. The similarity between the output distributions is quantified by two parameters $\varepsilon$ and $\delta$. We refer to the case when $\delta=0$ as \emph{pure} differential privacy, and when $\delta > 0$ we term it \emph{approximate} differential privacy.

\begin{definition}[Differential Privacy~\citep{dwork2006calibrating,dwork2006our}]
    A randomized algorithm $\A$ is said to be $(\varepsilon, \delta)$-differentially private if for any two neighboring  datasets $S_1$ and $S_2$ and any set $O$ of outputs, we have
    \begin{equation*}
        \Pr[\A(S_1)\in O]\le e^\varepsilon\Pr[\A(S_2)\in O] + \delta.
    \end{equation*}
    When $\delta = 0$, we may omit the parameter $\delta$ and simply say that $\A$ is $\varepsilon$-differentially private.
\end{definition}

In private learning, the learning algorithm produces a hypothesis that contains the prediction result for every $x\in\X$. However, in the private prediction problem introduced by~\citet{dwork2018privacy}, the algorithm $\A$ receives a dataset $S$ along with a query $x$ and outputs only the prediction on $x$. For privacy, we require it to satisfy differential privacy with respect to the dataset $S$. For utility, we treat $\A(S, \cdot)$ as a (randomized) classifier and require it to exhibit a low error in the PAC/agnostic setting. The formal definition is provided below.

\begin{definition}[Private Prediction~\cite{dwork2018privacy}]
    Let $\A$ be an algorithm that takes a labeled dataset $S$ and an unlabeled data point $x$ as input and produces a prediction value in $\{0, 1\}$. We say $\A$ is an $(\varepsilon,\delta)$-differentially private prediction algorithm if for any $x\in \X$, the output $A(S, x)$ is $(\varepsilon,\delta)$-differentially private with respect to $S$. 
    
    Define $\err_{\D}(\A(S, \cdot)) = \Pr_{(x,y)\sim\D, \A}[\A(S, x)\neq y]$. We say $\A$ $(\alpha, \beta)$-PAC learns $\C$ if for any distribution $\D$ such that $\Pr_{(x,y)\sim \D}[c(x) = y] = 1$ for some $c\in\C$, we have
    \begin{equation*}
        \Pr_{S\sim\D^n}[\err_{\D}(\A(S,\cdot)) \le \alpha]\ge 1-\beta.
    \end{equation*}
    Similarly, we say $\A$ $(\alpha, \beta)$-agnostically learns $\C$ if for any distribution $\D$, we have
    \begin{equation*}
        \Pr_{S\sim\D^n}[\err_{\D}(\A(S,\cdot)) \le \inf_{c\in\C}\err_{\D}(c)+\alpha]\ge 1-\beta.
    \end{equation*}
\end{definition}

We next describe the exponential mechanism~\citep{mcsherry2007mechanism} and its properties. Let $H$ be a finite set and $q:(\X\times\{0, 1\})^n\times H\to \mathbb{R}$ be a score function. We say $q$ has sensitivity $\Delta$ if $\max_{h\in H}\lvert q(S_1, h) - q(S_2, h) \rvert\le \Delta$ for any neighboring datasets $S_1$ and $S_2$ of size $n$. The exponential mechanism outputs an element $h\in H$ with probability
\begin{equation*}
    \frac{\exp(-\varepsilon\cdot q(S, h) / 2\Delta)}{\sum_{f\in H}\exp(-\varepsilon\cdot q(S, f) / 2\Delta)}.
\end{equation*}

\begin{lemma}[Properties of the Exponential Mechanism~\citep{mcsherry2007mechanism}]
\label{lem:prop_exp}
    The exponential mechanism is $\varepsilon$-differentially private. Moreover, with probability $1-\beta$, it outputs an $h$ such that
    \begin{equation*}
        q(S, h) \le \min_{f\in H}q(S, f)+\frac{2\Delta}{\varepsilon}\ln(\lvert H\rvert / \beta).
    \end{equation*}
\end{lemma}

An important property of differential privacy is that it implies generalization~\citep{dwork2015reusable,dwork2015preserving,bassily2016algorithmic,rogers2016max,feldman2017generalization,jung2020new}: a differentially private learning algorithm with a low empirical error also exhibits a low generalization error. The following version of the generalization property was used by~\citet{alon2020closure} in their transformation to handle improper learners. Note that this bound holds even for $\varepsilon > 1$.

\begin{lemma}[DP Generalization]
\label{lem:dp_gen}
    Let $\A$ be an $(\varepsilon,\delta)$-differentially private learning algorithm that takes $S_\X\in\X^n$ as input and outputs a predicate $h:\X\to \{0, 1\}$. Suppose each element in $S_\X$ is drawn i.i.d. from some distribution $\D_\X$, then we have
    \begin{equation*}
        \Pr\left[\mathbb{E}_{x\sim\D_\X}\left[h(x)\right] > e^{2\varepsilon}\left(\frac{\sum_{x\in S_\X}h(x)}{n}+\frac{10}{\varepsilon n}\log\left(\frac{1}{\varepsilon\delta n}\right)\right)\right] < O\left(\frac{\varepsilon\delta n}{\log\left(\frac{1}{\varepsilon\delta n}\right)}\right),
    \end{equation*}
    where the probability is taken over the random generation of $S_\X$ and the random coins of $\A$.
\end{lemma}

A learner that outputs a hypothesis with low empirical error is called an empirical learner~\citep{bun2015differentially}, which can be constructed from a PAC learner while preserving privacy.

\begin{definition}[PAC Empirical Learner]
    An algorithm $\A$ is said to be an $(\alpha, \beta)$-PAC empirical learner for concept class $\C$ with sample complexity $m$ if for any $c\in \C$ and any dataset $S=\{(x_1,c(x_1)),\allowbreak\dots,(x_m,c(x_m))\}$, it takes $S$ as input and outputs a hypothesis $h$ such that
    \begin{equation*}
        \Pr[\err_{S}(h)\le \alpha]\ge 1-\beta.
    \end{equation*}
\end{definition}

\begin{lemma}[Private Empirical Learner]
\label{lem:priv_emp}
    Let $\varepsilon \le 1$. Suppose $\A$ be an $(\varepsilon,\delta)$-differentially private $(\alpha, \beta)$-PAC learner for $\C$ with sample complexity $m$. Then there exists an $(1,O(\delta/\varepsilon))$-differentially private $(\alpha, \beta)$-PAC empirical learner $\A'$ for $\C$ with sample complexity $O(\varepsilon m)$. Moreover, if $\A$ is proper, then $\A'$ is also proper.
\end{lemma}

To prove the above lemma, we need the following result in~\citep{bun2015differentially}. We remark that although their original statement requires $n\ge 2m$, their proof actually works under the stronger conditions we present below.

\begin{lemma}
\label{lem:amplification}
    Let $\varepsilon\le 1$. Suppose $\A$ is an $(\varepsilon, \delta)$-differentially private algorithm that takes a dataset of size $m$ as input. For any $n$ such that $n\ge 2$ and $6\varepsilon m / n \le 1$, consider an algorithm $\A'$ works as follows:
    \begin{enumerate}
        \item Takes as input a dataset $S$ of size $n$.
        \item Constructs a dataset $T$ of size $m$, where each data point is sampled independently and uniformly from $S$ with replacement.
        \item Runs $\A$ on the dataset constructed in the previous step.
    \end{enumerate}
    Then $\A'$ is $(\varepsilon',\delta')$-differentially private for $\varepsilon' = 6\varepsilon m / n$ and $\delta'=\exp(6\varepsilon m /n)\frac{4m}{n}\delta$.
\end{lemma}

\begin{proof}[Proof of Lemma~\ref{lem:priv_emp}]
    Let $\A$ be an $(\varepsilon,\delta)$-differentially private $(\alpha, \beta)$-PAC learner for $\C$ with sample complexity $m$. Construct an algorithm $\A'$ as in Lemma~\ref{lem:amplification} with $n = \lceil6\varepsilon m\rceil$. Then Lemma~\ref{lem:amplification} directly implies that $\A'$ is $(1, O(\delta / \varepsilon))$-differentially private.

    Let $\D$ be the empirical distribution over $S$. Since $\A$ is an $(\alpha, \beta)$-PAC learner for $\C$, we have $\err_{\D}(\A(T)) \le \alpha$ with probability $1-\beta$ over the random generalization of $T$ and the internal randomness of $\A$. This is equivalent to $\Pr[\err_{S}(\A'(S))\le \alpha] \ge 1-\beta$, where the probability is taken over the internal randomness of $\A'$. Thus, $\A'$ is an $(\alpha, \beta)$-PAC empirical learner for $\C$.

    Moreover, since $\A'$ runs $A$ on some dataset $T$, $\A$ is proper implies that $\A'$ is proper.
\end{proof}
\section{The Transformation}

In this section, we present our realizable-to-agnostic transformation. We will start by describing a relabeling procedure proposed by~\citet{beimel2014learning}, which serves as the key component of our transformation.

Let $\C$ be a concept class and $S\in(\X\times\{0, 1\})^n$ be a dataset. In the agnostic setting, $S$ may not be consistent with any $c\in \C$. The idea of~\citet{beimel2014learning} is to first relabel $S$ by some concept $h\in \C$. After that, realizable learning algorithms can be applied.

Their method first constructs a candidate set $H$ such that for every labeling of $S$ there is one concept in $H$ consistent with that labeling. Then it initiates the exponential mechanism with score function $q(S, h) = \err_S(h)$ to select a concept $h$ for relabeling. Though the selection of $h$ is not private since the construction of $H$ depends on $S$, they proved that running a private algorithm on the relabeled dataset $S^h$ still preserves differential privacy. Moreover, the agnostic generalization bound and the property of the exponential mechanism ensure that $\err_{\D}(h)$ is close to the optimal error achieved by concepts in $\C$. Thus, if we can find some hypothesis $g$ such that $\err_{\D}(g)\approx \err_{\D}(h)$ by running a realizable learner on $S^h$, the resulting algorithm is an agnostic learner as desired.

However, the data points in $S^h$ are no longer i.i.d. from some distribution because the selection of $h$ depends on $S$. Therefore, directly running a private PAC learner on the relabeled dataset $S^h$ might not produce a good hypothesis. One should instead convert it to an empirical learner and apply the empirical learner to obtain some hypothesis $g$ whose empirical error is small on $S^h$. When the learner is proper (i.e., $g\in\C$), the realizable generalization bound implies that $\dis_{\D_\X}(g, h)$ is small, which indicates $\err_{\D}(g)\approx \err_{\D}(h)$. However, when the given private PAC learner is improper (and so is the resulting empirical learner), the realizable generalization bound cannot provide any guarantee on $\dis_{\D_\X}(g, h)$. 

To deal with improper learners, we next discuss a technique due to~\citet{alon2020closure}. In their work, they split the input dataset $S$ into two parts $S=U{\circ}V$ and showed that a simple variant of~\citet{beimel2014learning}'s algorithm, which outputs $V^h$ as well, still preserves privacy with respect to $U$ (i.e., the $V$ portion is regarded as public). They then considered an auxiliary algorithm that outputs $g\oplus \bar{h}$ for some $\bar{h}\in\C$ consistent with $V^h$ and used the generalization property of DP to derive an upper bound on $\dis_{\D_\X}(g, \bar{h})$. Since $\dis_{\D_\X}(h, \bar{h})$ can be controlled by applying the realizable generalization bound to $V_{\X}$, the triangle inequality yields a bound on $\dis_{\D_\X}(g, h)$. Therefore, the generalization error of $g$ can be successfully bounded.

The above transformation incurs an extra sample complexity of $\widetilde{O}(\vc(\C)/\alpha^2)$ when converting a private PAC learner to a private agnostic learner~\citep{beimel2014learning,alon2020closure}. However, it only provides a constant level of privacy (i.e., $\varepsilon=\Theta(1)$) even if the PAC learner $\A$ is $(\varepsilon,\delta)$-differentially private for some $\varepsilon\ll 1$. To achieve an arbitrary privacy level $\varepsilon$, one has to apply the privacy-amplification-by-subsampling technique~\citep{kasiviswanathan2011can}: first subsample a dataset $T$ from $S$ with size $\lvert T\rvert\approx\varepsilon \lvert S\rvert =\varepsilon n$, then perform the transformation on $T$ only. Since the size of $T$ should be at least $\widetilde{O}(\vc(\C)/\alpha^2)$ to ensure the agnostic generalization of every $h\in H\subseteq \C$, the overall transformation results in an extra sample complexity of $\widetilde{O}(\vc(\C)/\alpha^2\varepsilon)$.

We now illustrate how to eliminate the $1/\varepsilon$ factor. Let $W$ denote the dataset containing the data points that are not in $T$, i.e., $S=T{\circ}W$. In the above process, we discard $W$ after sampling and do not exploit any information contained in $W$, which seems too wasteful. Our idea is to utilize $W$ so that we can apply the agnostic generalization bound to the entire dataset rather than $T$ only. To be specific, we still construct $H$ from the subsampled dataset $T$, but evaluate the score function of every $h\in H$ over $S$. Thus, we only require $\lvert S\rvert \ge\widetilde{O}(\vc(\C)/\alpha^2)$ to ensure agnostic generalization.

The primary obstacle here is how to ensure that such a modification still preserves privacy. Let $S_1=T_1{\circ}W_1$ and $S_2=T_2{\circ}W_2$ be two neighboring datasets. There are two cases: $T_1=T_2$ and $T_1\neq T_2$. In the case where $T_1 =T_2$, we will construct the same candidate set $H$ from them. Therefore, it is easy to achieve privacy for the selection of $h$ (the hypothesis for relabeling the dataset) by launching an $\varepsilon$-differentially private exponential mechanism since $W_1$ and $W_2$ are neighboring datasets. According to the post-processing property of DP, running any algorithm on the relabeled dataset is also $\varepsilon$-differentially private.

However, it is not as simple in the case where $T_1\neq T_2$, as the candidate sets constructed from $T_1$ and $T_2$ are different. To see why, let $\hat{T}$ be the overlapping portion of $T_1$ and $T_2$, which has size $\lvert T\rvert - 1$. The privacy analysis proposed by~\citet{beimel2014learning} requires $\lvert q(S_1, h_1) - q(S_2,h_2)\rvert$ to be small for any $h_1$ and $h_2$ that agree on $\hat{T}$. This naturally holds in the original transformation, where the score function is $q(S, h) = \err_{T}(h)$ (recall that we discard all data points in $W$) and the difference $\lvert q(S_1, h_1) - q(S_2,h_2)\rvert$ is only $1/\lvert T \rvert$ since $h_1$ and $h_2$ agree on $\hat{T}$. However, if we try to incorporate the $W$ portion and set the score function to be $q(S, h) = \err_S(h)$, the difference can be close to $1$ as $h_1$ and $h_2$ may totally disagree on $W_1=W_2$, failing to provide a satisfactory privacy guarantee.

We overcome this issue by devising a score function that estimates the generalization error of $h$ using the entire dataset while having a small ``sensitivity''. In particular, we run the exponential mechanism with the following score function:

\begin{equation*}
    q(T{\circ}W, h) = \min_{f\in \C}\{\dis_{T_\X}(h, f) + \err_{W}(f)\}.
\end{equation*}

The above score function can be interpreted as searching for a concept $f$ that is close to $h$ over $T_\X$ and also has a low empirical error on $W$. We describe the relabeling procedure in Algorithm~\ref{alg:relabel}, where we set the privacy parameter as $\varepsilon$ and sensitivity parameter as $1/\lvert W\rvert$ to ensure that it preserves $\varepsilon$-differential privacy when $T_1=T_2$. In the case that $T_1\neq T_2$, one can verify that $\lvert q(S_1, h_1) - q(S_2, h_2)\rvert\le 1/ \lvert T\rvert= \Theta(1 / \varepsilon n)$. Because we have set the privacy parameter as $\varepsilon$ and sensitivity parameter $\Delta = \Theta(1/n)$, we can apply the analysis of~\citet{beimel2014learning} to show that running a private algorithm on the relabeled dataset is still private with a constant privacy parameter. Note that this case only happens with probability $\varepsilon$. We can apply the privacy-amplification-by-subsampling argument to deduce (actually, the formal proof requires a more delicate privacy analysis for this case) that the overall algorithm is private with privacy parameter $\varepsilon$. We formally describe the details of the entire agnostic learning algorithm in Algorithm~\ref{alg:agnostic} and state its privacy guarantee in the following lemma.

\begin{lemma}[Privacy of $\Agn$]
\label{lem:priv_agn}
    Suppose $\A$ is $(1, \delta)$-differentially private. Then $\Agn$ (Algorithm~\ref{alg:agnostic}) is $(O(\varepsilon),O(\varepsilon\delta))$-differentially private.
\end{lemma}

\begin{proof}
    Let $S_1,S_2$ be two neighboring datasets and $O$ be any set of outputs. Without loss of generality, we assume $S_1$ and $S_2$ differ on the first element. That is, $S_1=\{(x_1,y_1),(x_2,y_2),\dots,(x_n,y_n)\}$ and $S_2=\{(x'_1,y'_1),(x_2,y_2),\dots,\allowbreak(x_n,y_n)\}$. Define
    \begin{equation*}
        p_t(I) = \Pr[\Agn(S_t)\in O\mid \text{the sampled index set is }I]
    \end{equation*}
    for $t\in\{1, 2\}$ and $I\subseteq [n]$ of size $\lceil \varepsilon n\rceil$. Since $I$ is sampled uniformly at random, we have 
    \begin{equation*}
        \Pr[\Agn(S_t)\in O] = \frac{1}{\binom{n}{\lvert I\rvert}}\sum_{I}p_t(I).
    \end{equation*}

    Now consider a fixed index set $I$. Let $T_1,W_1$ and $T_2,W_2$ be the corresponding partitions of $S_1$ and $S_2$. We will consider two cases: $1\in I$ and $1\notin I$.

    When $1\notin I$, we have $T_1=T_2$. Therefore, $\Rel(T_1,W_1)$ and $\Rel(T_2,W_2)$ will construct the same candidate set $H$. For every $h\in H$, suppose its score function is minimized by $f_1$ on dataset $T_1{\circ}W_1$, i.e., $q(T_1{\circ}W_1, h) = \dis_{(T_1)_\X}(h, f_1) + \err_{W_1}(f_1)$. Since $W_1$ and $W_2$ are neighboring datasets, we can bound $q(T_2{\circ}W_2,h)$ as follows:
    \begin{align*}
        q(T_2{\circ}W_2, h) &= \min_{f\in \C}\left\{ \dis_{(T_2)_\X}(h, f) + \err_{W_2}(f) \right\}\\&\le \dis_{(T_2)_\X}(h, f_1) + \err_{W_2}(f_1) \\&\le \dis_{(T_1)_\X}(h, f_1) + \err_{W_1}(f_1) + \frac{1}{ \lvert W_1\rvert} \\&= q(T_1{\circ}W_1, h) + \frac{1}{ \lvert W_1\rvert}.
    \end{align*}
    By symmetry, $q(T_2{\circ}W_2, h) \le q(T_1{\circ}W_1, h) + \frac{1}{\lvert W_1\rvert}$. Therefore, the sensitivity of $q$ is $\frac{1}{\lvert W_1\rvert}$. It then follows by Lemma~\ref{lem:prop_exp} that 
    \begin{equation*}
        \Pr[\Rel(T_1,W_1) = T_1^h] \le e^\varepsilon\Pr[\Rel(T_2,W_2) = T_2^h]
    \end{equation*}
    for any $T_1^h=T_2^h$. The post-processing property of DP immediately implies
    \begin{equation*}
        p_1(I)\le e^{\varepsilon}p_2(I).
    \end{equation*}

    We then turn to the case that $1\in I$. We will prove the following conclusion for every $i\notin I$:
    \begin{equation*}
        p_1(I)\le e^{O(1)}p_2((I\setminus\{1\})\cup\{i\})+O(\delta).
    \end{equation*}

    Let $T_2'$ and $W_2'$ be the partitions of $S_2$ using $(I\setminus \{1\})\cup \{i\}$ as the index set. Since $1\in I$, we have $T_1\setminus \{(x_1, y_1)\} = T_2' \setminus \{(x_i, y_i)\} = \hat{T}$ for some $\hat{T}$ of size $\lvert I\rvert - 1$. Let $H_1$ and $H_2'$ denote the candidate sets constructed during the execution of $\Rel(T_1, W_1)$ and $\Rel(T_2', W_2')$. For each possible labeling $\hat{T}^c$ of $\hat{T}_\X$, define $P_1(c) = \left\{f\in H_1:\err_{\hat{T}^c}(f) = 0\right\}$ and $P_2'(c) = \left\{f\in H_2':\err_{\hat{T}^c}(f) = 0\right\}$, i.e., the sets consisting of hypotheses in $H_1$ and $H_2'$ that agree with $c$ on $\hat{T}_\X$. Since the label set is $\{0, 1\}$, we have $1\le \lvert P_1(c)\rvert, \lvert P_2'(c)\rvert \le 2$.

    We next pick arbitrary $h_1\in P_1(c)$ and $h_2'\in P_2'(c)$ and compare their scores. Suppose $q(T_1{\circ}W_1, h_1)$ is minimized by $f_1$. Note that $W_1$ and $W_2'$ are neighboring datasets (since $W_1\setminus \{(x_i, y_i)\} = W_2' \setminus \{(x_1', y_1')\}$), and $h_1$ and $h_2'$ agree on $\hat{T}_\X$, we have
    \begin{align*}
        q(T_2'{\circ}W_2', h_2') &= \min_{f\in \C} \left\{\dis_{(T_2')_\X}(h_2', f) + \err_{W_2'}(f) \right\}\\&\le \dis_{(T_2')_\X}(h_2', f_1) + \err_{W_2'}(f_1) \\&\le \dis_{(T_1)_\X}(h_1, f_1) + \frac{1} { \lvert T_1\rvert} + \err_{W_1}(f_1) + \frac{1} { \lvert W_1\rvert} \\&= q(T_1{\circ}W_1, h_1) + \frac{1}{  \lvert T_1\rvert }+ \frac{1 }{ \lvert W_1\rvert}.
    \end{align*}
    Since $\lvert T_1\rvert\ge n\varepsilon$ and $\lvert W_1\rvert \le n - n\varepsilon$, we have $\frac{\lvert W_1\rvert }{ \lvert T_1\rvert} \le \frac{1-\varepsilon}{  \varepsilon}$. Thus, 
    \begin{align*}
        \exp(-\varepsilon\cdot q(T_2'{\circ}W_2', h_2')/2\Delta) &\ge \exp\left(-\varepsilon\cdot \left(q(T_1{\circ}W_1, h_1) + \frac{1}{ \lvert T_1\rvert} + \frac{1}{ \lvert W_1\rvert}\right)/2\Delta\right) \\
        &= \exp(-\varepsilon\cdot q(T_1{\circ}W_1, h_1)/2\Delta)\cdot\exp\left(-\frac{\varepsilon}{2}\cdot\left(\frac{\lvert W_1\rvert}{\lvert T_1\rvert} + \frac{\lvert W_1\rvert}{\lvert W_1\rvert}\right)\right) \\
        &\ge \exp(-\varepsilon\cdot q(T_1{\circ}W_1, h_1)/2\Delta)\cdot \exp\left(-\frac{\varepsilon}{2}\cdot\left(\frac{1-\varepsilon}{\varepsilon} + 1\right) \right) \\
        &= \exp(-\varepsilon\cdot q(T_1{\circ}W_1, h_1)/2\Delta)\cdot \exp(-1/2).
    \end{align*}

    By symmetry (because the above analysis only relies on the facts that $W_1$ and $W_2'$ are neighboring and $h_1$ and $h_2'$ agree on $\hat{T}_\X$), we have 
    \begin{equation*}\exp(-\varepsilon\cdot q(T_1{\circ}W_1, h_1)/2\Delta)\ge \exp(-\varepsilon\cdot q(T_2'{\circ}W_2', h_2')/2\Delta)\cdot \exp(-1/2).
    \end{equation*}
    Then, the fact that $1\le \lvert P_1(c)\rvert, \lvert P_2'(c)\rvert \le 2$ gives
    \begin{equation*}
        \sum_{h_1\in P_1(c)} \exp(-\varepsilon\cdot q(T_1{\circ}W_1, h_1)/2\Delta) \ge \frac{1}{2}\sum_{h_2'\in P_2'(c)} \exp(-\varepsilon\cdot q(T_2'{\circ}W_2', h_2')/2\Delta)\cdot \exp(-1/2).
    \end{equation*}
    Summing over all hypotheses in $H_1$, we get
    \begin{align*}
        \sum_{f\in H_1}\exp(-\varepsilon\cdot q(T_1{\circ}W_1, f)/2\Delta) &= \sum_{\hat{T}^c}\sum_{h_1\in P_1(c)} \exp(-\varepsilon\cdot q(T_1{\circ}W_1, h_1)/2\Delta) \\
        &\ge \sum_{\hat{T}^c}\frac{1}{2}\sum_{h_2'\in P_2'(c)}\exp(-\varepsilon\cdot q(T_2'{\circ}W_2', h_2')/2\Delta)\cdot \exp(-1/2) \\
        &=\frac{1}{2\sqrt{e}}\sum_{f\in H_2'}\exp(-\varepsilon\cdot q(T_2'{\circ}W_2', f)/2\Delta).
    \end{align*}

    Note that $(T_1)^{h_1}$ and $(T_2')^{h_2'}$ are neighboring datasets (since $h_1$ and $h_2'$ agree on $\hat{T}_\X$). Then by the fact that $\A$ is $(1,\delta)$-differentially private, we have

    \begin{align*}
        &\Pr[\Rel(T_1, W_1) = (T_1)^{h_1}] \cdot \Pr[\A((T_1)^{h_1})\in O] \\
        ={}& \frac{\exp(-\varepsilon\cdot q(T_1{\circ}W_1, h_1)/2\Delta)}{\sum_{f\in H_1}\exp(-\varepsilon\cdot q(T_1{\circ}W_1, f)/2\Delta)}\cdot \Pr[\A((T_1)^{h_1})\in O] \\
        \le{}& 2e\cdot \frac{\exp(-\varepsilon\cdot q(T_2'{\circ}W_2', h_2')/2\Delta)}{\sum_{f\in H_2'}\exp(-\varepsilon\cdot q(T_2'{\circ}W_2', f)/2\Delta)}\cdot (e\cdot \Pr[\A((T_2')^{h_2'})\in O]+\delta) \\
        ={}& 2e\cdot\Pr[\Rel(T_2', W_2') = (T_2')^{h_2'}]\cdot(e\cdot \Pr[\A((T_2')^{h_2'})\in O]+\delta).
    \end{align*}
    
    We can then bound $p_1(I)$ as follows:
    \begin{align*}
        p_1(I) &= \Pr[\A(\Rel(T_1, W_1))\in O] \\
        &= \sum_{\hat{T}^c}\sum_{h_1\in P_1(c)}\Pr[\Rel(T_1, W_1) = (T_1)^{h_1}]\cdot\Pr[\A((T_1)^{h_1})\in O] \\
        &\le \sum_{\hat{T}^c}2\sum_{h_2'\in P_2'(c)}2e\cdot\Pr[\Rel(T_2', W_2') = (T_2')^{h_2'}]\cdot(e\cdot\Pr[\A((T_2')^{h_2'})\in O] + \delta)\\
        &= 4e\cdot\left(e\cdot \Pr[\A(\Rel(T_2', W_2'))\in O] + \delta\right)\\ 
        &= e^{2+2\ln 2}p_2((I\setminus\{1\})\cup\{i\}) + 4e\delta.
    \end{align*}

    Note that the summation $\sum_{I:1\in I}\sum_{i\in[n]\setminus I}p_2((I\setminus\{1\})\cup\{i\})$ actually counts every $p_2(I)$ (where $1\notin I$) exactly $\lvert I\rvert$ times. Thus,
    \begin{align*}
        \sum_{I:1\in I}p_1(I) &= \frac{1}{n - \lvert I\rvert}\sum_{I:1\in I}\sum_{i\in[n]\setminus I}p_1(I) \\
        &\le \frac{1}{n - \lvert I\rvert}\sum_{I:1\in I}\sum_{i\in[n]\setminus I}\left[e^{2+2\ln 2}p_2((I\setminus\{1\})\cup\{i\})+4e\delta\right] \\
        &= \frac{\lvert I\rvert}{n - \lvert I \rvert}\sum_{I:1\notin I}e^{2+2\ln 2}p_2(I)+4e\delta\cdot \binom{n - 1}{ \lvert I\rvert - 1} \\
        &= O(\varepsilon)\sum_{I:1\notin I}p_2(I) + O(\delta)\cdot \binom{n - 1}{ \lvert I\rvert - 1},
    \end{align*}
    where in the last line we use the fact that 
    \begin{equation*}
        \frac{\lvert I\rvert}{n - \lvert I\rvert} = \frac{\lceil \varepsilon n\rceil}{n - \lceil \varepsilon n\rceil} \le \frac{2 \varepsilon n}{n - 2\varepsilon n} \le 6\varepsilon
    \end{equation*}
    assuming $\varepsilon n\ge 1$ and $\varepsilon \le 1/3$. This implies that
    \begin{align*}
        \Pr[\Agn(S_1)\in O] &= \frac{1}{\binom{n}{ \lvert I\rvert}}\left(\sum_{I:1\notin I} p_1(I) + \sum_{I:1\in I}p_1(I)\right)\\
        &\le \frac{1}{\binom{n}{\lvert I\rvert}}\left(e^{\varepsilon}\sum_{I:1\notin I}p_2(I) + O(\varepsilon)\sum_{I:1\notin I}p_2(I) + O(\delta)\cdot \binom{n - 1}{\lvert I\rvert - 1}\right) \\
        &\le (e^{\varepsilon} + O(\varepsilon))\cdot \frac{1}{\binom{n}{\lvert I\rvert}}\sum_{I} p_2(I) + O(\delta)\cdot\frac{\lvert I\rvert}{n}\\
        &\le e^{O(\varepsilon)}\Pr[\Agn(S_2)\in O] + O(\varepsilon\delta).
    \end{align*}
\end{proof}

\begin{algorithm}[t]
    \label{alg:relabel}
    \caption{$\Rel$}
    \DontPrintSemicolon 
    \KwInput{Parameter $\varepsilon$, Datasets $T,W$}
        Initialize $H = \emptyset$\;
        For every possible labeling in $\Pi_{\C}(T_\X)$, add to $H$ an arbitrary concept $h\in \C$ that is consistent with the labeling\;
        Define the following score function $q$:
        \begin{equation*}
            q(T{\circ}W, h) = \min_{f\in \C} \left\{\dis_{T_\X}(h, f) + \err_{W}(f)\right\}
        \end{equation*}\\
        Choose $h\in H$ using the exponential mechanism with privacy parameter $\varepsilon$, score function $q$, and sensitivity parameter $\Delta=\frac{1}{\lvert W\rvert}$\;
    Relabel $T$ using $h$ and output the relabeled dataset $T^h$\;
\end{algorithm}

\begin{algorithm}[t]
    \label{alg:agnostic}
    \caption{$\Agn$}
    \DontPrintSemicolon 
    \KwInput{Parameter $\varepsilon$, Dataset $S=\{(x_1,y_1),\dots,(x_n,y_n)\}$, Private Algorithm $\A$}
        Sample a subset $I\subseteq [n]$ of size $\lvert I\rvert = \lceil \varepsilon n\rceil$ uniformly at random\;
        Let $T=\{(x_i, y_i)\mid i\in I\}$ and $W=\{(x_i,y_i)\mid i\in [n]\setminus I\}$\;
        Execute $\Rel$ (Algorithm~\ref{alg:relabel}) with parameter $\varepsilon$ and input datasets $T, W$ to obtain relabeled dataset $T^h$\;
        Output $\A(T^h)$\;
\end{algorithm}

We then prove the utility guarantee of $\Agn$. The first step is to show that $\Rel$ will relabel the dataset using some $h$ whose generalization error is close to the optimal. We show that it suffices to set $\lvert T\rvert = \widetilde{O}(\vc(\C)/\alpha)$ and $\lvert W\rvert = \widetilde{O}(\vc(\C)\cdot\max(1/\alpha^2,1/\alpha\varepsilon))$.

\begin{claim}
\label{cla:util_rel}
    Let $T$ and $W$ be two datasets with every data point sampled i.i.d. from $\D$. Suppose
    \begin{equation*}
        \lvert T\rvert\ge C_1\cdot \frac{\vc(\C)\ln(1/\alpha) + \ln(1/\beta)}{\alpha}
    \end{equation*}
    and
    \begin{equation*}
        \lvert W\rvert\ge \max\left( C_2\cdot \frac{\vc(\C) + \ln(1/\beta)}{\alpha^2}, \frac{\lvert T\rvert}{6\varepsilon}\right),
    \end{equation*}
    where $C_1$ and $C_2$ are universal constants. Then with probability $1-\beta$, $\Rel$ (Algorithm~\ref{alg:relabel}) will relabel $T$ using some $h\in \C$ such that
    \begin{equation*}
        \err_{\D}(h) \le \inf_{c\in \C} \err_\D(c) + \alpha.
    \end{equation*}
\end{claim}

We use the following technical lemma~\citep{anthony1999neural} in bounding the sample complexity incurred by the exponential mechanism.
\begin{lemma}
\label{lem:vc_tech}
Let $d\ge 1$ and $\alpha,\beta\in(0, 1)$. Then if $n\ge \frac{2d\ln(2/\alpha) + 2\ln(1/\beta)}{\alpha}$, we have
\begin{equation*}
    n\alpha \ge d\ln\left(\frac{en}{d}\right) + \ln\left(\frac{1}{\beta}\right).
\end{equation*}
\end{lemma}

\begin{proof}[Proof of Claim~\ref{cla:util_rel}]
    Define the following three events:
    \begin{itemize}
        \item $E_1$: For every $c\in\C$, it holds that $\lvert\err_{\D}(c) - \err_{W}(c)\rvert \le \alpha/9$.
        \item $E_2$: The exponential mechanism chooses an $h\in H$ such that \begin{equation*}q(T{\circ}W,h)\le \min_{f\in H}q(T{\circ}W,f) + \alpha/9.\end{equation*}
        \item $E_3$: For any $h_1, h_2\in\C$ such that $\dis_{T_\X}(h_1, h_2)\le \alpha/3$, it holds that $\dis_{\D_\X}(h_1, h_2)\le 2\alpha/3$.
    \end{itemize}

    We first show that $T$ will be relabeled by some $h$ such that $\err_{\D}(h)\le \inf_{c\in\C}\err_\D(c) + \alpha$ if the above events happen. Let $\eta = \inf_{c\in\C}\err_\D(c)$. Let $f_0\in\C$ be some concept that minimizes the empirical error on $W$, i.e., $\err_{W}(f_0) = \min_{c\in\C}\err_{W}(c)$. Then $E_1$ implies that $\err_{W}(f_0) \le \inf_{c\in\C}\err_{W}(c) \le \inf_{c\in\C}\err_{\D}(c) + \alpha/9 = \eta + \alpha/9$. Since every labeling in $\Pi_{\C}(T_\X)$ is labeled by some $h\in H$, there exists some $h_0\in H$ that agrees with $f_0$ on $T_\X$. Thus,
    \begin{align*}
        q(T{\circ}W, h_0) &= \min_{f\in\C}\left\{\dis_{T_\X}(h_0, f) + \err_W(f)\right\} \\
        &\le \dis_{T_\X}(h_0, f_0) + \err_W(f_0) \\
        &\le \eta + \alpha / 9.
    \end{align*}
    Then, event $E_2$ ensures that the exponential mechanism outputs some $h\in H$ such that
    \begin{align*}
        q(T{\circ}W,h) &\le \min_{f\in H}q(T{\circ}W, f) + \alpha / 9\\
        &\le q(T{\circ}W, h_0) + \alpha/ 9 \\
        &\le \eta + 2\alpha / 9.
    \end{align*}

    Suppose $q(T{\circ}W, h) = \dis_{T_\X}(h, f) + \err_W(f)$ for some $f\in \C$. Event $E_1$ ensures that
    \begin{equation*}
        \err_{\D}(f)\le \err_W(f) + \alpha / 9 \le q(T{\circ}W, h) + \alpha /9\le \eta + \alpha / 3.
    \end{equation*}
    Moreover, since $\err_{\D}(f) \ge \eta$, we have
    \begin{align*}
        \dis_{T_\X}(h, f) &= q(T{\circ}W, h) - \err_W(f) \\ &\le q(T{\circ}W, h) - (\err_{\D}(f) - \alpha/9) \\ &\le q(T{\circ}W, h) - \eta + \alpha / 9 \\ &\le \eta + 2\alpha / 9- \eta + \alpha / 9\\ &= \alpha / 3.
    \end{align*}
    Event $E_3$ then ensures that $\dis_{\D_\X}(h, f) \le 2\alpha /3$. By the triangle inequality, we obtain
    \begin{equation*}
        \err_{\D}(h) \le \err_{\D}(f) + \dis_{\D_\X}(h, f) \le \eta + \alpha/3 + 2\alpha/3 = \eta + \alpha.
    \end{equation*}

    To complete the proof, we now show $E_1\cap E_2\cap E_3$ happens with probability $1-\beta$. By Lemma~\ref{lem:agn_gen}, $E_1$ happens with probability $1-\beta / 3$ given that $\lvert W\rvert \ge C\cdot \frac{\vc(\C) + \ln(1/\beta)}{\alpha^2}$ for some constant $C$.

    We then consider $E_2$. By Sauer's Lemma (Lemma~\ref{lem:sauer}), we have $\lvert H\rvert \le \left(\frac{e\lvert T\rvert}{\vc(\C)}\right)^{\vc(\C)}$. Then by Lemma~\ref{lem:prop_exp}, with probability $1-\beta/3$, the exponential mechanism selects some $h$ such that
    \begin{align*}
        q(T{\circ}W,h) &\le \min_{f\in H}q(T{\circ}W, f) + \frac{2}{\lvert W\rvert\varepsilon}\ln(\lvert H\rvert/\beta) \\
        &\le \min_{f\in H}q(T{\circ}W, f) + \frac{12}{\lvert T\rvert}\left(\vc(\C)\ln\left(\frac{e\lvert T\rvert}{\vc(\C)}\right) + \ln(1/\beta)\right) \\
        &\le \min_{f\in H}q(T{\circ}W, f) + \alpha / 9,
    \end{align*}
    where in the second inequality we use $\lvert W\rvert\ge \frac{\lvert T\rvert}{6\varepsilon}$ and in the last inequality we apply Lemma~\ref{lem:vc_tech}, which requires $\lvert T\rvert\ge C'\cdot\frac{\vc(\C)\ln(1/\alpha) + \ln(1/\beta)}{\alpha}$.
    This means $E_2$ happens with probability $1-\beta /3$.

    Finally, Lemma~\ref{lem:rea_gen} implies that event $E_3$ happens with probability $1-\beta / 3$ given that $ \lvert T\rvert\ge C''\cdot\frac{\vc(\C)\ln(1/\alpha) + \ln(1/\beta)}{\alpha}$. Therefore, we have $\Pr[E_1\cap E_2\cap E_3] \ge 1-\beta$ by the union bound.
\end{proof}

Now it remains to show that the output hypothesis $\A(T^h)$ is close to $h$ on the underlying distribution $\D$. When $\A$ is a proper learner (i.e., $\A(T^h)\in\C$), this is directly implied by the realizable generalization.

To handle the case that $\A$ may be improper, we adopt the proof strategy of~\citet{alon2020closure}. The key idea is to construct an auxiliary algorithm $\Aux$ (Algorithm~\ref{alg:aux}). It splits $T$ into $T=U{\circ}V$ (we set $\lvert U \rvert=\lvert V \rvert = \lvert T\rvert/2$) and outputs $\A(T^h)\oplus \bar{h}$ for some $\bar{h}\in\C$ that is consistent with $V^h$. Given that $\lvert V\rvert$ is sufficiently large, the realizable generalization bound implies that $\bar{h}$ is close to $h$ on the underlying distribution $\D$, hence on $U_\X$ as well. Since the output hypothesis $\A(T^h)$ is also close to $h$ over $U_\X$, $\A(T^h)$ should be close to $\bar{h}$ on $U_\X$.

\begin{algorithm}[t]
    \label{alg:aux}
    \caption{$\Aux$}
    \DontPrintSemicolon 
    \KwInput{Parameter $\varepsilon$, Datasets $U, V, W$, Private Algorithm $\A$}
        Execute $\Rel$ (Algorithm~\ref{alg:relabel}) with parameter $\varepsilon$ and input datasets $T=U{\circ}V, W$ to obtain relabeled dataset $T^h=U^h{\circ}V^h$\;
        Select an arbitrary $\bar{h}\in \C$ that is consistent with $V^h$\;
        Output $\A(T^h)\oplus\bar{h}$\;
\end{algorithm}

We prove that $\Aux$ satisfies a constant level of differential privacy using an argument similar to part of the proof of Lemma~\ref{lem:priv_agn}. This allows us to bound the generalization disagreement between $\A(T^h)$ and $\bar{h}$ by the generalization property of DP. Therefore, we have $\err_{\D}(\A(T^h))\approx\err_{\D}(\bar{h})\approx\err_{\D}(h)$, which proves the utility guarantee of $\Agn$.

\begin{lemma}[Utility of $\Agn$]
\label{lem:util_agn}
    Suppose $\A$ is a $(1, \delta)$-differentially private $(\alpha, \beta)$-PAC empirical learner with sample complexity $m$. Then $\Agn$ (Algorithm~\ref{alg:agnostic}) is an $(O(\alpha), O(\beta + \varepsilon n\delta))$-agnostic learner with sample complexity
    \begin{equation*}
        n = O\left(\frac{m}{\varepsilon} + \frac{\vc(\C)\log(1/\alpha) + \log(1/\beta)}{\alpha\varepsilon} + \frac{\vc(\C) + \log(1/\beta)}{\alpha^2}\right).
    \end{equation*}
\end{lemma}

To prove Lemma~\ref{lem:util_agn}, we follow the proof strategy of~\citet{alon2020closure}. We first construct an auxiliary algorithm ($\Aux$) and prove the following claim, which allows us to employ the generalization property of DP. The proof is analogous to part of the proof of Lemma~\ref{lem:priv_agn}.

\begin{claim}
\label{cla:priv_aux}
    Suppose $\A$ is $(1, \delta)$-differentially private. For any public $V$ and $W$, $\Aux(U, V, W)$ is $(2+2\ln2, 4e\delta)$-differentially private with respect to $U$.
\end{claim}
\begin{proof}
    Let $U_1$ and $U_2$ be two neighboring datasets and $O$ be any set of the outputs of $\Aux$. Then $T_1=U_1{\circ}V$ and $T_2=U_2{\circ}V$ are also neighboring datasets. Therefore, there is some $\hat{T}$ of size $\lvert T_1 \rvert - 1$ such that $T_1\setminus\{(x_1,y_1)\} = T_2\setminus \{(x_1',y_1')\}$ for data points $(x_1,y_1)\in T_1$ and $(x_1',y_1')\in T_2$. Let $H_t$ be the candidate set constructed during the execution of $\Rel(T_t, W)$, where $t\in\{1, 2\}$. For each possible labeling $\hat{T}^c$ of $\hat{T}_\X$, let $P_t(c)$ be the set of hypotheses added to $H_t$ in the execution of $\Rel(T_t, W)$. Pick arbitrary $h_1\in P_1(c)$ and $h_2\in P_2(c)$ and suppose $q(T_1{\circ}W, h_1)$ is minimized by $f_1$, we have 
    \begin{align*}
        q(T_2{\circ}W, h_2) &= \min_{f\in\C}\left\{\dis_{(T_2)_\X}(h_2, f) + \err_{W}(f)\right\} \\
        &\le \dis_{(T_2)_\X}(h_2, f_1) + \err_{W}(f_1) \\
        &\le \dis_{(T_1)_\X}(h_1, f_1) + \frac{1}{\lvert T_1\rvert}+ \err_{W}(f_1) \\
        &=q(T_1{\circ}W, h_1) + \frac{1}{\lvert T_1\rvert}.
    \end{align*}
    Since $\lvert T_1\rvert\ge \varepsilon n$ and $\lvert W\rvert \le n - \varepsilon n$, we have $\frac{\lvert W\rvert }{ \lvert T_1\rvert} \le \frac{1-\varepsilon}{ \varepsilon}$. Thus, 
    \begin{align*}
        \exp(-\varepsilon\cdot q(T_2{\circ}W,h_2) / 2\Delta) &\ge \exp\left(-\varepsilon\cdot\left(q(T_1{\circ}W, h_1) + \frac{1}{\lvert T_1\rvert}\right) / 2\Delta\right) \\
        &= \exp(-\varepsilon\cdot q(T_1{\circ}W,h_1)/2\Delta)\cdot\exp\left(-\frac{\varepsilon\lvert W\rvert}{2\lvert T_1\rvert}\right) \\
        &\ge \exp(-\varepsilon\cdot q(T_1{\circ}W,h_1)/2\Delta)\cdot\exp\left(-\frac{1-\varepsilon}{2}\right) \\
        &\ge \frac{1}{\sqrt{e}}\cdot \exp(-\varepsilon\cdot q(T_1{\circ}W,h_1)/2\Delta).
    \end{align*}
    By symmetry, we also have
    \begin{equation*}
        \exp(-\varepsilon\cdot q(T_1{\circ}W,h_1) / 2\Delta)\ge \frac{1}{\sqrt{e}}\cdot \exp(-\varepsilon\cdot q(T_2{\circ}W,h_2) / 2\Delta).
    \end{equation*}
    The fact that $1\le \lvert P_1(c)\rvert,\lvert P_2(c)\rvert \le 2$ gives
    \begin{equation*}
        \sum_{h_1\in P_1(c)} \exp(-\varepsilon\cdot q(T_1{\circ}W, h_1)/2\Delta) \ge \frac{1}{2\sqrt{e}}\sum_{h_2\in P_2(c)} \exp(-\varepsilon\cdot q(T_2{\circ}W, h_2)/2\Delta).
    \end{equation*}
    Therefore,
    \begin{align*}
        \sum_{f\in H_1}\exp(-\varepsilon\cdot q(T_1{\circ}W, f)/2\Delta) &= \sum_{\hat{T}^c}\sum_{h_1\in P_1(c)} \exp(-\varepsilon\cdot q(T_1{\circ}W, h_1)/2\Delta) \\
        &\ge \sum_{\hat{T}^c}\frac{1}{2}\sum_{h_2\in P_2(c)}\exp(-\varepsilon\cdot q(T_2{\circ}W, h_2)/2\Delta)\cdot \exp(-1/2) \\
        &=\frac{1}{2\sqrt{e}}\sum_{f\in H_2}\exp(-\varepsilon\cdot q(T_2{\circ}W, f)/2\Delta).
    \end{align*}

    Since $h_1$ and $h_2$ agree on $\hat{T}_\X$, $(T_1)^{h_1}$ and $(T_2)^{h_2}$ are neighboring datasets. Moreover, note that $h_1$ and $h_2$ agree on $V_\X$ because $V_\X$ is just part of $\hat{T}_\X$. Therefore, $\Aux(U_1,V,W)$ and $\Aux(U_2,V,W)$ will select the same $\bar{h}$ from $V^{h_1}=V^{h_2}$. By the fact that $\A$ is $(1,\delta)$-differentially private, we obtain
    \begin{align*}
        &\Pr[\Rel(T_1, W) = (T_1)^{h_1}] \cdot \Pr[\A((T_1)^{h_1})\oplus \bar{h}\in O] \\
        ={}& \frac{\exp(-\varepsilon\cdot q(T_1{\circ}W, h_1)/2\Delta)}{\sum_{f\in H_1}\exp(-\varepsilon\cdot q(T_1{\circ}W, f)/2\Delta)}\cdot \Pr[\A((T_1)^{h_1})\oplus \bar{h}\in O] \\
        \le{}& 2e\cdot \frac{\exp(-\varepsilon\cdot q(T_2{\circ}W, h_2)/2\Delta)}{\sum_{f\in H_2}\exp(-\varepsilon\cdot q(T_2{\circ}W, f)/2\Delta)}\cdot (e\cdot \Pr[\A((T_2)^{h_2})\oplus \bar{h}\in O]+\delta) \\
        ={}& 2e\cdot\Pr[\Rel(T_2, W) = (T_2)^{h_2}]\cdot(e\cdot \Pr[\A((T_2)^{h_2})\oplus \bar{h}\in O]+\delta).
    \end{align*}

    Let $\bar{h}=\bar{h}(V^h)$ denote the selection rule of $\bar{h}$. Summing over all labelings gives
    \begin{align*}
        &\Pr[\Aux(U_1, V, W)\in O] \\
        ={}& \sum_{\hat{T}^c}\sum_{h_1\in P_1(c)}\Pr[\Rel(T_1, W) = (T_1)^{h_1}]\cdot\Pr[\A((T_1)^{h_1})\oplus\bar{h}(V^{h_1})\in O] \\
        \le{}& \sum_{\hat{T}^c}2\sum_{h_2\in P_2(c)}2e\cdot\Pr[\Rel(T_2, W) = (T_2)^{h_2}]\cdot(e\cdot\Pr[\A((T_2)^{h_2})\oplus\bar{h}(V^{h_2})\in O] + \delta)\\
        ={}& 4e\cdot\left(e\cdot \Pr[\Aux(U_2, V, W)\in O] + \delta\right)\\ 
        ={}& e^{2+2\ln 2}\Pr[\Aux(U_2, V, W)\in O] + 4e\delta.
    \end{align*}
\end{proof}

\begin{proof}[Proof of Lemma~\ref{lem:util_agn}]
    Recall that
    \begin{equation*}
        n = O\left(\frac{m}{\varepsilon} + \frac{\vc(\C)\log(1/\alpha) + \log(1/\beta)}{\alpha\varepsilon} + \frac{\vc(\C) + \log(1/\beta)}{\alpha^2}\right).
    \end{equation*}
    Moreover, assuming $\varepsilon \le 1/3$ and $\varepsilon n\ge 1$, we have
    \begin{equation*}
        \frac{\lvert W \rvert}{\lvert T\rvert} = \frac{n - \lceil \varepsilon n\rceil}{\lceil \varepsilon n\rceil} \ge \frac{n-2 \varepsilon n}{2 \varepsilon n} \ge \frac{1}{6\varepsilon}.
    \end{equation*}
    Therefore, it is not hard to verify that the conditions in Claim~\ref{cla:util_rel} are fulfilled. Thus, with probability $1-\beta$, $T$ will be relabeled by some $h$ with $\err_{\D}(h) \le \inf_{c\in \C} \err_\D(c) + \alpha$. 
    
    Since $\A$ is an $(\alpha,\beta)$-PAC empirical learner, the final output hypothesis $g=\A(T^h)$ satisfies $\err_{T^h}(g) \le \alpha$ with probability $1-\beta$. This is equivalent to $\dis_{T_\X}(h, g)\le \alpha$. Suppose we choose $\lvert U\rvert = \lvert V\rvert = \lvert T\rvert / 2$ in $\Aux$. Note that $h$ and $\bar{h}$ agree on $V_\X$, the realizable generalization property (Lemma~\ref{lem:rea_gen}) implies that $\dis_{\D_\X}(h, \bar{h}) \le \alpha$ with probability $1-\beta$. Applying Lemma~\ref{lem:rea_gen} again (over $T_\X$), we have $\dis_{T_\X}(h, \bar{h})\le 2\alpha$ with probability $1-\beta$. Therefore, \begin{equation*}
        \dis_{U_\X}(g, \bar{h})\le \dis_{T_\X}(g, \bar{h}) \le \dis_{T_\X}(g, h) + \dis_{T_\X}(h, \bar{h}) \le 3\alpha.
    \end{equation*}

    We then bound the generalization disagreement between $g$ and $\bar{h}$ using the generalization property of DP. By Claim~\ref{cla:priv_aux}, $\Aux$ is $(O(1), O(\delta))$-differentially private with respect to $U$. Therefore, it is also $(O(1), O(\delta + \beta / \lvert U\rvert))$-differentially private. Then by Lemma~\ref{lem:dp_gen} and the fact that $\lvert U\rvert =\lvert T\rvert / 2\ge C\ln(1/\beta)/\alpha$ for some large constant $C$, we have
    \begin{align*}
        \dis_{\D_\X}(g, \bar{h}) &\le O\left(\dis_{U_\X}(g, \bar{h}) + \frac{1}{\lvert U\rvert}\log\left(\frac{1}{\delta\lvert U\rvert+\beta}\right)\right) \\
        &\le O\left(\alpha + \frac{1}{\lvert U\rvert}\log\left(\frac{1}{\beta}\right)\right) \\
        &\le O(\alpha)
    \end{align*}
    with probability $1-O(\delta\lvert U\rvert + \beta)$.

    Putting all things together, the union bound ensures that with probability $1-O(\delta\lvert U\rvert + \beta) = 1 - O(\varepsilon\delta n+\beta)$, we have
    \begin{align*}
        \err_{\D}(g) &\le \err_\D(h) + \dis_{\D_\X}(h, g) \\
        &\le \inf_{c\in \C} \err_\D(c) + \alpha + \dis_{\D_\X}(h, \bar{h}) +  \dis_{\D_\X}(g, \bar{h}) \\
        &\le \inf_{c\in \C} \err_\D(c) + O(\alpha).\\
    \end{align*}
\end{proof}

Now we are able to prove our main theorem. In our transformation, we first invoke Lemma~\ref{lem:priv_emp} to create a $(1,\delta'=O(\delta / \varepsilon))$-differentially private empirical learner from the given $(\varepsilon,\delta)$-differentially private PAC learner. We then use this empirical learner as the private algorithm $\A$ in $\Agn$. Lemma~\ref{lem:priv_agn} ensures that $\Agn$ is $(O(\varepsilon), O(\varepsilon\delta')=O(\delta))$-differentially private. The final sample complexity follows from Lemma~\ref{lem:util_agn}.

\begin{theorem}
\label{thm:trans}
    Let $\varepsilon \le O(1)$. Suppose there is an $(\varepsilon, \delta)$-differentially private $(\alpha, \beta)$-PAC learner for $\C$ with sample complexity $m$. Then there exists an $(\varepsilon,\delta)$-differentially private $(O(\alpha),O(\beta+\delta n))$-agnostic learner for $\C$ with sample complexity
    \begin{equation*}
        n = O\left(m + \frac{\vc(\C)\log(1/\alpha) + \log(1/\beta)}{\alpha\varepsilon} + \frac{\vc(\C) + \log(1/\beta)}{\alpha^2}\right).
    \end{equation*}
    Moreover, if the original learner is proper, then the resulting learner is also proper.
\end{theorem}
\begin{proof}
    By Lemma~\ref{lem:priv_emp}, there exists a $(1, \delta')$-differentially private $(\alpha,\beta)$-PAC empirical learner $\A$ for $\C$ with sample complexity $O(\varepsilon m)$, where $\delta' = O(\delta / \varepsilon)$. Then by Lemma~\ref{lem:priv_agn}, we have that $\Agn$ is $(O(\varepsilon), O(\varepsilon\delta')=O(\delta))$-differentially private. Moreover, by Lemma~\ref{lem:util_agn}, $\Agn$ is an $(O(\alpha), O(\beta+\varepsilon \delta' n)=O(\beta + \delta n))$-agnostic learner with sample complexity
    \begin{equation*}
        n = O\left(m + \frac{\vc(\C)\log(1/\alpha) + \log(1/\beta)}{\alpha\varepsilon} + \frac{\vc(\C) + \log(1/\beta)}{\alpha^2}\right).
    \end{equation*}

    The constant factors on the privacy parameters can be removed by the privacy-amplification-by-subsampling technique. Furthermore, if the original learner is proper, then Lemma~\ref{lem:priv_emp} suggests that $\A$ is also proper. Therefore, $\Agn$ is proper since it outputs $\A(T^h)$.
\end{proof}

Ignoring logarithmic factors, the second term in the resulting sample complexity is dominated by $m$~\citep{bun2015differentially}. Hence, the extra sample complexity introduced by our transformation is $\widetilde{O}(\vc(\C)/\alpha^2)$, which is near-optimal.
\section{Private Prediction}

In this section, we provide an algorithm for private prediction in the agnostic setting with a nearly optimal sample complexity of $\widetilde{O}(\vc(\C)/\alpha\varepsilon + \vc(\C)/\alpha^2)$. We need the following algorithm for private prediction in the realizable setting~\citep{dwork2018privacy}.
\begin{lemma}
\label{lem:priv_pred}
    Let $r=\lceil 6\ln(4/\alpha)/\varepsilon\rceil$ and $\A'$ be a PAC learning algorithm. Suppose hypotheses $g_1,\dots, g_r$ are obtained by running $r$ instances of $\A'$ on disjoint portions of the input dataset $S$. Then there exists an $\varepsilon$-differentially private prediction algorithm $\A$ that answers a prediction query $x$ based on some aggregation mechanism over $\{g_1(x),\dots,g_r(x)\}$ such that
    \begin{equation*}
        \forall i\in[r],\err_{\D}(g_i)\le \alpha/4 \Rightarrow \err_{\D}(\A(S,\cdot))\le \alpha.
    \end{equation*}
\end{lemma}

Like private agnostic learning, we run $\Agn$ with the private prediction algorithm $\A$ in the above lemma. The privacy analysis remains unchanged. However, the utility guarantee no longer holds because $\A$ only preserves privacy for a single prediction rather than the entire hypothesis, prohibiting us from applying the generalization property of DP.

To circumvent this issue, we choose the base learner $\A'$ to be an algorithm that outputs some concept in $\C$ that is consistent with the (relabeled) dataset. This allows us to leverage the realizable generalization bound to argue that with high probability, the generalization error (with respect to the relabeled distribution) of every instance of $\A'$ is small no matter which concept $h\in\C$ is selected to relabel the dataset, indicating that the generalization error of $\A$ is also small. That is to say, the generalization error of $\A$ is close to that of $h$ with respect to the original distribution.

\begin{theorem}[Theorem~\ref{thm:pred_intro} Restated]
\label{thm:pred}
    Let $\varepsilon \le O(1)$. There exists an $\varepsilon$-differentially private prediction algorithm that $(\alpha, \beta)$-agnostically learns $\C$ with sample complexity
    \begin{equation*}
        O\left(\frac{\vc(\C)\log^2(1/\alpha)+\log(1/\alpha)\log(\log(1/\alpha)/\beta)}{\alpha\varepsilon} + \frac{\vc(\C)+\log(1/\beta)}{\alpha^2}\right).
    \end{equation*}
\end{theorem}
\begin{proof}
    Consider an algorithm $\A'$ that arbitrarily selects a concept from $\C$ that is consistent with the input dataset. Let $\A$ be a $1$-differentially private mechanism constructed using Lemma~\ref{lem:priv_pred} with the base learner $\A'$. Then by Lemma~\ref{lem:priv_agn}, the overall algorithm $\Agn$ is $O(\varepsilon)$-differentially private.

    We now prove the accuracy of $\Agn$. Let $\D$ be the underlying distribution and $\D_\X$ be its marginal distribution on $\X$. Set $n' \ge  C\cdot \frac{\vc(\C)\ln(1/\alpha) + \ln(r/\beta)}{\alpha}$ for some constant $C$ and $r=\lceil 6\ln(4/\alpha)\rceil$. Let $T_\X'=\{x_1,\dots,x_{n'}\}$ be some unlabeled dataset, where each $x_i$ is drawn i.i.d. from $D_\X$. When $C$ is sufficiently large, the realizable generalization bound (Lemma~\ref{lem:rea_gen}) suggests that with probability $1-\beta / r$, it holds that $\dis_{\D_\X}(h_1, h_2)\le \alpha / 4$ for any $h_1,h_2\in\C$ that are consistent on $T_\X'$. Let $S=T{\circ}W$ be the entire input dataset such that
    \begin{equation*}
        \lvert T\rvert\ge rn' = O\left(\frac{\vc(\C)\log^2(1/\alpha) + \log(1/\alpha)\log(\log(1/\alpha)/\beta)}{\alpha}\right)
    \end{equation*} and $h$ be the concept chosen by $\Rel$ for relabeling. Then by Lemma~\ref{lem:priv_pred} and the union bound, it holds with probability $1-\beta$ that $\err_{\D^h}(\A(T^h,\cdot))\le \alpha$, where $\D^h$ is the distribution obtained by labeling $\D_\X$ with $h$.

    By Claim~\ref{cla:util_rel}, we have $\err_{\D}(h)\le \inf_{c\in\C}\err_{\D}(c) + \alpha$ with probability $1-\beta$ given that $\lvert T\rvert\ge C_1\cdot\frac{\vc(\C)\ln(1/\alpha) + \ln(1/\beta)}{\alpha}$ and $\lvert W\rvert\ge C_2\cdot\frac{\vc(\C)+\ln(1/\beta)}{\alpha^2} + \frac{\lvert T\rvert}{6\varepsilon}$ for constants $C_1,C_2$. Therefore, the triangle inequality and the union bound imply that with probability $1-2\beta$, we have
    \begin{equation*}
        \err_{\D}(\Agn(S,\cdot)) = \err_{\D}(\A(T^h,\cdot))\le \err_{\D^h}(\A(T^h,\cdot)) + \err_{\D}(h)\le \inf_{c\in\C}\err_{\D}(c) + 2\alpha.
    \end{equation*}

    Adjusting $\varepsilon,\alpha,\beta$ by constant factors yields the desired result.
\end{proof}

\section*{Acknowledgements}
We thank anonymous COLT reviewers for helpful comments. The research was supported in part by a RGC RIF grant under contract R6021-20, RGC TRS grant under contract T43-513/23N-2, RGC CRF grants under contracts C7004-22G, C1029-22G and C6015-23G, NSFC project under contract 62432008, and RGC GRF grants under contracts 16207922, 16207423 and 16211123.
{

\bibliographystyle{plainnat}
\bibliography{references.bib}

}

\appendix
\section{Discussion on the Computational Complexity}

The relabeling procedure in our transformation has two main steps:
\begin{enumerate}
    \item Constructing a candidate set that contains all the labelings.
    \item Running the exponential mechanism.
\end{enumerate}
By Sauer's Lemma, the size of the candidate set is $O(n^{\vc(\C)})$, where $n$ is the sample size. If we do not require the time complexity to be polynomial in $\vc(\C)$ (i.e., we treat $\vc(\C)$ as a fixed constant), step 2 can be done with a polynomial number of calls of an ERM oracle. Although step 1 requires enumerating $2^n$ labelings for general classes, it can be done efficiently for classes with certain structures (e.g., point functions, thresholds, and axis-aligned rectangles). In these cases, the reduction is efficient. In summary, the transformation is inefficient in general, but can be made efficient for special classes.

We remark that there are classes (e.g., halfspaces) that are (computationally) easy to learn in the realizable setting but hard to learn in the agnostic setting under computational or cryptographic assumptions~\citep{feldman2009agnostic}. Since our algorithm is a reduction from agnostic learning to realizable learning, we cannot hope that it will be generally efficient given these hardness results. Hence, we focus on information-theoretic bounds in this work.

\end{document}